\documentclass[11pt]{article}

\usepackage[preprint]{acl}

\usepackage{times}
\usepackage{latexsym}

\usepackage[T1]{fontenc}

\usepackage[utf8]{inputenc}

\usepackage{microtype}

\usepackage{inconsolata}

\usepackage{graphicx}

\usepackage{enumerate}

\usepackage{amsmath}
\usepackage{amssymb}
\usepackage{mathtools}
\usepackage{amsthm}
\usepackage{amsmath}
\usepackage{algorithm}
\usepackage{algorithmic}
\usepackage{booktabs}
\usepackage{microtype}
\usepackage{graphicx}
\usepackage{subcaption}
\usepackage{booktabs} 
\usepackage[table]{xcolor} 
\usepackage{adjustbox}

\usepackage{float}
\theoremstyle{plain}
\newtheorem{theorem}{Theorem}[section]

\newtheorem{lemma}[theorem]{Lemma}

\theoremstyle{definition}

\theoremstyle{remark}
\newtheorem{remark}[theorem]{Remark}

\newcommand{\method}{\textsc{Rewind}}
\newcommand{\methodfull}{REWIND}

\newenvironment{appendixtableblock}[1][\columnwidth]{%
  \begingroup
  \small
  \setlength{\tabcolsep}{4pt}%
  \renewcommand{\arraystretch}{1.05}%
  \begin{adjustbox}{max width=#1}%
}{%
  \end{adjustbox}%
  \endgroup
}

\newenvironment{enumeratesquish}[2]{
  \begin{list}{\labelenumi}{
    \setlength{\itemsep}{#1}
    \setlength{\labelwidth}{#2}
    \setlength{\leftmargin}{\labelwidth}
    \addtolength{\leftmargin}{\labelsep}
    \setlength{\parsep}{0em}
    \setlength{\topsep}{0em}
  }
}{\end{list}}

\title{REWIND: Uncertainty-Guided Trajectory Repair for Test-Time Reasoning}

\author{
Yi Ding\textsuperscript{1} \quad
Mengyu Wang\textsuperscript{2} \quad
Min Zhou\textsuperscript{3} \quad
Menglin Yang\textsuperscript{1,*} \\[0.4em]
{\small\normalfont\textsuperscript{1}The Hong Kong University of Science and Technology (Guangzhou)} \\
{\small\normalfont\textsuperscript{2}University of Edinburgh \quad \textsuperscript{3}Huawei} \\[0.2em]
{\small\normalfont\textsuperscript{*}Corresponding author: \href{mailto:menglinyang@hkust-gz.edu.cn}{menglinyang@hkust-gz.edu.cn}}
}

\hypersetup{pdftitle={REWIND: Uncertainty-Guided Trajectory Repair for Test-Time Reasoning}, pdfauthor={Yi Ding, Mengyu Wang, Min Zhou, Menglin Yang}}

\begin{document}
\maketitle
\begin{abstract}

An early mistake in the generated reasoning can constrain all subsequent reasoning. Repeated-generation methods, like MCMC-based approaches, attempt to improve output quality by sampling multiple generations, but their effectiveness can be limited by randomly selecting where to regenerate. On MATH500 with Qwen2.5-Math-7B, we find that incorrect responses have substantially higher mean token entropy than correct responses, with an average within-question relative increase of approximately 112\%, suggesting that entropy can signal reasoning errors. Motivated by this finding, we introduce \methodfull{}, a training-free controller that uses entropy dynamics to propose targeted suffix regenerations. Across three model families and five benchmarks, \method{} outperforms all baselines, with gains of 0.51–10.00 percentage points over the strongest baseline. These results demonstrate the effectiveness of targeted, uncertainty-guided trajectory repair across models and tasks. Our code is available at \url{https://anonymous.4open.science/r/rewind-D310}.

\end{abstract}

\section{Introduction}
\label{sec:intro}

Large language models (LLMs) can solve multi-step reasoning problems, but an early incorrect premise or calculation can constrain the rest of a generated solution \cite{wei2022chain,dziri2023faith}. To mitigate such issues caused by sequential decoding, explicit revision methods discard part of that context and give the model another opportunity to continue from an earlier prefix. However, the central question of revision is how to choose a useful cut point without a trained verifier or external feedback.

Recent works have proposed many different approaches to respond to this question. Iterative self-feedback and sampling-based refinement provide ways to revise a draft \cite{madaan2023selfrefine,karan2025reasoning}. MCMC-based Power Sampling chooses a random cut point to regenerate the suffix and then accepts or rejects the proposal using model likelihoods. Its cut-point policy is independent of local confidence and uncertainty. A cut after a mistaken premise retains it, whereas a cut far before it can replace substantial correct work. Since a filter acts only after re-generation, it can accept or reject a completed proposal but cannot guide the re-generation process itself. Therefore, revision-site selection becomes a practical bottleneck, motivating trajectory diagnostics that direct proposals toward candidate regions of instability.

Although continuation sampling is stochastic, its final acceptance is based on the generation statistics, since confidence and uncertainty statistics can contain information about downstream failure \cite{kadavath2022language,zhang2025entropy}. Motivated by this, there may exist signals during the generation process indicating potential re-generation points. Shannon entropy summarizes the full next-token distribution, capturing competition among alternatives, and can be measured from model logits without an additional verifier or reference answers \cite{shannon1948mathematical}. We examine this signal through 20 ordinary Qwen2.5-Math-7B generations per MATH500 question. Among the 463 questions with both correct and incorrect responses under the original boxed-answer evaluator, we find that the mean token entropy of incorrect responses is approximately \textbf{112\% higher} than that of correct responses (Section~\ref{subsec:entropy_observations}). Moreover, on questions with fewer incorrect responses, those errors exhibit higher entropy, suggesting that entropy can identify mistakes that are more likely to be corrected through regeneration. This token-level generation-error association motivates uncertainty-conditioned revision.

Therefore, we propose \methodfull{} (Reasoning Error Windback via INtrinsic Detection), a training-free, entropy-conditioned suffix-proposal controller. A high-entropy peak can occur after an earlier unstable choice has influenced later steps. The controller therefore searches immediately before the peak for the largest discrete entropy increase,  and aligns the cut with a natural delimiter to regenerate the suffix. Then an MH-style score filter decides whether to retain the proposal. This makes the re-generation more targeted to cover the potential mistakes rather than random selection, which may miss critical errors.

REWIND consistently improves reasoning performance across three model families and five benchmarks, achieving the highest accuracy among the baselines. It yields gains of 0.51–10.00 percentage points, with the largest improvement on AIME2025 using DeepSeek-R1-Distill-Qwen-1.5B (DeepSeek-1.5B), from 16.67\% to 26.67\%. Controlled ablations show that entropy-gradient backtracking and boundary alignment both improve the performance. Further evaluations demonstrate that the improvements transfer to few-shot prompting and code generation and remain stable across random seeds. Together, these results establish REWIND as an effective approach to targeted test-time trajectory repair.

In summary, our contributions are as follows:

\begin{enumeratesquish}{0em}{0.7em}
    \item[1.] \textbf{An empirical entropy--correctness association.} Through repeated generations on MATH500, we find a clear association between a question's error frequency and the token entropy of its responses, indicating that entropy provides a useful signal of potential generation errors.
    \item[2.] \textbf{A training-free trajectory-repair method.} We introduce \method{}, which locates entropy-gradient spikes before uncertainty peaks, aligns revision points with natural delimiters, and filters regenerated suffixes using an MH-style score, without additional training or an external verifier.
    \item[3.] \textbf{Comprehensive empirical evaluation.} Experiments across three model families and five benchmarks show that \method{} achieves the best performance among all baselines, improving over random-proposal MCMC by 0.51--10.00 percentage points. Controlled ablations and evaluations on few-shot reasoning and code generation further support the method's design and its transfer across tasks.
\end{enumeratesquish}

\section{Related Work}
\label{sec:related_work}
\paragraph{Training-free test-time scaling.}

Existing inference-time reasoning methods improve model outputs through either candidate exploration or iterative refinement. Self-consistency aggregates independently sampled solutions \cite{wang2023selfconsistency}, while Tree of Thoughts and MCTS-based approaches explicitly explore alternative reasoning branches \cite{yao2023tree,zhou2023language,liu2025monte}. In contrast, iterative refinement methods revise existing outputs rather than exploring independent reasoning paths. Self-Refine uses model-generated textual feedback to guide revisions \cite{madaan2023selfrefine}, although intrinsic self-correction can be unreliable without informative feedback \cite{huang2023large}. Sampling-based methods offer another refinement mechanism.
Early work such as CGMH applies Metropolis--Hastings sampling
to constrained sentence generation through word-level edits
\cite{miao2019cgmh}, while block MH methods propose
whole-sequence rewrites using a language model
\cite{forristal2023block}.
More recently, Power Sampling employs MH-based suffix
proposals to favor sequences under a sharpened model
distribution \cite{karan2025reasoning}.
Building on suffix-based refinement, \method{} shifts the focus from how candidate revisions are generated to where they should be applied, enabling more targeted refinement.

\paragraph{Uncertainty-guided control.}
Model-derived signals have been used to guide inference at different levels. At the token level, DoLa contrasts predictions across layers, while typical decoding constrains token information content \cite{chuang2024dola,meister2023locally}. At the response level, adaptive consistency uses uncertainty to allocate additional samples \cite{aggarwal2023let}, while semantic entropy quantifies uncertainty across semantically equivalent generations \cite{farquhar2024detecting,kuhn2023semantic}. At the search level, Entro-duction uses entropy and varentropy to decide whether to deepen, expand, or terminate reasoning exploration \cite{zhang2025entropy}.

Recent work connects token-level uncertainty with the branching of generation trajectories. \citet{bigelow2024forking} identify forking tokens by resampling alternative continuations and examining their downstream outcomes. \citet{wang2025beyond} further associate high-entropy tokens with reasoning forks and exploit these positions to focus RLVR policy-gradient updates. At inference time, DeepConf \citep{fu2026deep} aggregates token confidence to filter or terminate reasoning traces. \method{} uses local uncertainty to guide repair within an existing response: entropy-gradient backtracking and boundary alignment select a cut point for suffix regeneration, without updating model parameters.

\section{Preliminaries and Observations}
\label{sec:prelim_observations}

We first review Power Sampling to establish the technical background. We then present an empirical study of token entropy across repeated reasoning generations, demonstrating its potential as a during-generation signal for detecting errors.

\subsection{Power Sampling for LLMs}
\label{subsec:power_sampling}
For a question $q$, an autoregressive LLM defines a distribution over responses $\mathbf{y}$ through
\begin{equation}
 p_\theta(\mathbf{y}\mid q)
 =\prod_{t=1}^{T_{\mathbf{y}}}
 p_\theta(y_t\mid q,y_{<t}),
\end{equation}
where $T_{\mathbf{y}}$ includes the terminating token. Power Sampling~\cite{karan2025reasoning} favors higher-likelihood complete responses using
\begin{equation}
\begin{aligned}
 \pi_\alpha(\mathbf{y}\mid q)
 &=\frac{p_\theta(\mathbf{y}\mid q)^\alpha}{Z_\alpha(q)},\\
 Z_\alpha(q)&=\sum_{\mathbf{y}'}p_\theta(\mathbf{y}'\mid q)^\alpha,
 \qquad \alpha\geq1.
\end{aligned}
\label{eq:power_target}
\end{equation}
The normalization is over complete responses. Consequently, this target differs from applying temperature $1/\alpha$ separately to each next-token distribution. Because $Z_\alpha(q)$ is intractable, a Metropolis--Hastings (MH) update can use a proposal kernel $r(\mathbf{y}'\mid\mathbf{y},q)$ and accept with probability
\begin{equation}
 A_q(\mathbf{y},\mathbf{y}')=
 \min\!\left\{1,
 \frac{p_\theta(\mathbf{y}'\mid q)^\alpha
 r(\mathbf{y}\mid\mathbf{y}',q)}
 {p_\theta(\mathbf{y}\mid q)^\alpha
 r(\mathbf{y}'\mid\mathbf{y},q)}\right\}.
\label{eq:power_mh}
\end{equation}
The unknown normalizer cancels. A random-suffix proposal chooses a cut point and autoregressively regenerates the continuation. Equation~\ref{eq:power_mh} gives the generic MH rule and Section~\ref{sec:method} specifies our state-dependent proposal and implemented MH-style filter.

\subsection{Empirical Entropy--Error Associations}
\label{subsec:entropy_observations}
To examine the relationship between token entropy and response correctness, we collect repeated generations for each MATH500 question. Let $\mathcal{Q}$ denote the set of 500 questions. For each $q\in\mathcal{Q}$, we sample $N=20$ responses from Qwen2.5-Math-7B. Let $\mathbf{y}_{q, i}$ denote the $i$-th response for question $q$, $\ell_{q, i}$ its token length, and $c_{q, i}\in\{0,1\}$ its correctness label under the original boxed-answer evaluator. 
Each response consists of a sequence of generated tokens, $\mathbf{y}_{q,i}=(t^0_{q,i},t^1_{q,i},\ldots,t^{\ell_{q,i}-1}_{q,i})$. For each position $j$, we record the full-vocabulary next-token distribution before applying sampling transformations. Let $\mathcal{V}$ denote the vocabulary. The distribution $p_{q,i,j}$ and its entropy $H_{q,i,j}$ at position $j$ are defined as:
\begin{equation}
\begin{aligned}
p_{q,i,j}(v)
&= p_\theta\bigl(v \mid q,\mathbf{y}_{q,i,<j}\bigr),\\
H_{q,i,j}
&= -\sum_{v\in\mathcal{V}}p_{q,i,j}(v)\log p_{q,i,j}(v),
\end{aligned}
\label{eq:observation_token_entropy}
\end{equation}
where $\mathbf{y}_{q,i,<j}=(t^0_{q,i},\ldots,t^{j-1}_{q,i})$ is the generated prefix before position $j$, and is empty when $j=0$.

For $s\in\{+,-\}$, let $\mathcal{I}_q^s$ denote the indices of responses in class $s$, and let $n_q^s=\sum_{i\in\mathcal{I}_q^s}\ell_{q,i}$ be the total number of tokens in that class for question $q$. The mean entropy for class $s$ is
\begin{equation}
h_q^s
=
\frac{1}{n_q^s}
\sum_{i\in\mathcal{I}_q^s}
\sum_{j=0}^{\ell_{q,i}-1}
H_{q,i,j},
\qquad n_q^s>0.
\label{eq:question_class_entropy}
\end{equation}
Thus, $h_q^s$ is the mean entropy pooled over all tokens from responses in class $s$ for question $q$.

\paragraph{Paired relative increase.}
The paired set is
\begin{equation}
 \mathcal{Q}_{\pm}
 =\{q\in\mathcal{Q}:n_q^+>0,\ n_q^->0,\ h_q^+>0\}.
\end{equation}
There are 463 paired questions. For each, define
\begin{equation}
 r_q=\frac{h_q^- - h_q^+}{h_q^+},\qquad
 \overline r=\frac{1}{|\mathcal{Q}_{\pm}|}
                  \sum_{q\in\mathcal{Q}_{\pm}}r_q.
\label{eq:mean_question_relative_increase}
\end{equation}
The mean per-question relative increase is approximately \textbf{112\%} (median: \textbf{83.61\%}). Incorrect-response mean token entropy exceeds the correct-response mean on \textbf{88.12\%} of paired questions, with an increase of \textbf{at least 70\% on 55.51\%}.

\paragraph{Response-level diagnostics.}
To complement the within-question analysis, we test whether entropy summaries of individual responses can distinguish incorrect from correct answers. For each nonempty response, we compute its peak token entropy, its largest adjacent-token entropy increase, and its mean token entropy. Table~\ref{tab:entropy_diagnostics} reports the average value of each statistic for correct and incorrect responses, along with its AUROC for ranking incorrect responses above correct ones. All three statistics are higher on average for incorrect responses, and mean token entropy provides the strongest discrimination (AUROC 0.746). These results further show that token-level entropy can help distinguish correct from incorrect responses.

\begin{table}[t]
\centering
\caption{\textbf{Entropy diagnostics on MATH500.} The same Qwen2.5-Math-7B study yields 9,951 nonempty responses after excluding 49 empty ones. Scores exclude prompt, padding, and EOS. AUROC predicts incorrect original labels. Brackets give 95\% CIs from 2,000 whole-question bootstrap resamples.}
\label{tab:entropy_diagnostics}
\footnotesize
\setlength{\tabcolsep}{2pt}
\begin{adjustbox}{max width=\columnwidth}
\begin{tabular}{@{}lrrr@{}}
\toprule
Signal & \shortstack{Correct\\mean} & \shortstack{Incorrect\\mean} & \shortstack{AUROC\\{}[95\% CI]} \\
\midrule
Peak entropy & 3.147 & 3.787 & 0.677 [0.660, 0.694] \\
Max entropy gradient & 1.912 & 2.655 & 0.724 [0.711, 0.736] \\
Mean token entropy & 0.177 & 0.371 & \textbf{0.746 [0.731, 0.760]} \\
\bottomrule
\end{tabular}
\end{adjustbox}
\end{table}

\paragraph{Entropy across success-count groups.} 
We further examine whether the entropy of incorrect responses varies with how often a question is answered correctly. We group the questions by their number of correct responses among 20 generations and compare the mean entropy of incorrect responses across groups. For each question $q\in\mathcal{Q}$, let $m_q=\sum_{i=1}^{20}c_{q,i}$ denote the number of correct responses.
We group consecutive counts in pairs, keeping 0 and 19 separate:
\begin{equation}
 \mathcal{B}=\bigl(\{0\},\{1,2\},\ldots,\{17,18\},\{19\}\bigr).
\end{equation}
For each bin $B\in\mathcal{B}$, let $\mathcal{G}_B$ be the set of questions whose number of correct responses falls in $B$, and let $\mathcal{G}_B^-$ be the subset with at least one incorrect response. We define $y_B$ as the average, across questions in $\mathcal{G}_B^-$, of the mean token entropy of incorrect responses:
\begin{equation}
\begin{aligned}
 y_B&=\frac{1}{|\mathcal{G}_B^-|}
             \sum_{q\in\mathcal{G}_B^-}h_q^-.
\end{aligned}
\label{eq:incorrect_group_bar}
\end{equation}
Based on the calculated \(y_B\) values, Figure~\ref{fig:entropy_incorrect_groups} shows the mean entropy of incorrect responses across success-count groups. The clear upward trend indicates that questions answered correctly more often tend to have higher-entropy incorrect responses. Therefore, higher entropy may help identify errors that are more likely to be corrected through re-generation.

\begin{figure}[t]
\centering
\includegraphics[width=\columnwidth]{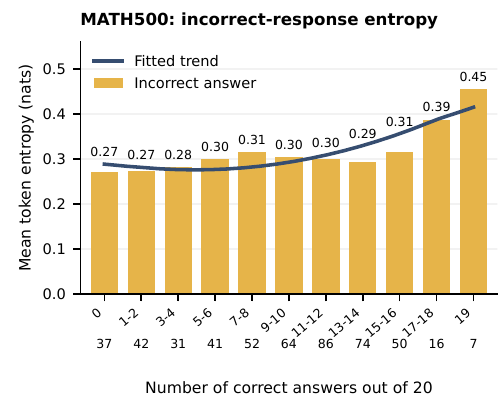}
\caption{\textbf{Incorrect-response entropy on MATH500.} The x-axis groups questions by the number of correct responses out of 20. Yellow bars show equal-question means of incorrect-response entropy (Equation~\ref{eq:incorrect_group_bar}). Numbers above bars give entropy in nats; numbers below bin labels count questions in each group.}
\label{fig:entropy_incorrect_groups}
\end{figure}

This group-level analysis complements the paired and response-level comparisons and further motivates the uncertainty-conditioned proposals developed next. These analyses establish that entropy carries information about errors at the response level; whether it also identifies useful revision \emph{sites} is an empirical question we test directly through proposal ablations in Section~\ref{subsec:mechanism_analysis}.

\section{Method}
\label{sec:method}
Motivated by the observed entropy patterns, we propose \method{}, an entropy-guided method for selecting where to revise a reasoning trajectory. Figure~\ref{fig:2} summarizes the proposal policy. Here, $\mathbf{x}$ denotes the concatenation of the prompt and response, and $L_{\text{context}}$ is the prompt length. Only generated positions are eligible for revision.

\begin{figure*}[t]
    \centering
    \includegraphics[width=0.9\linewidth]{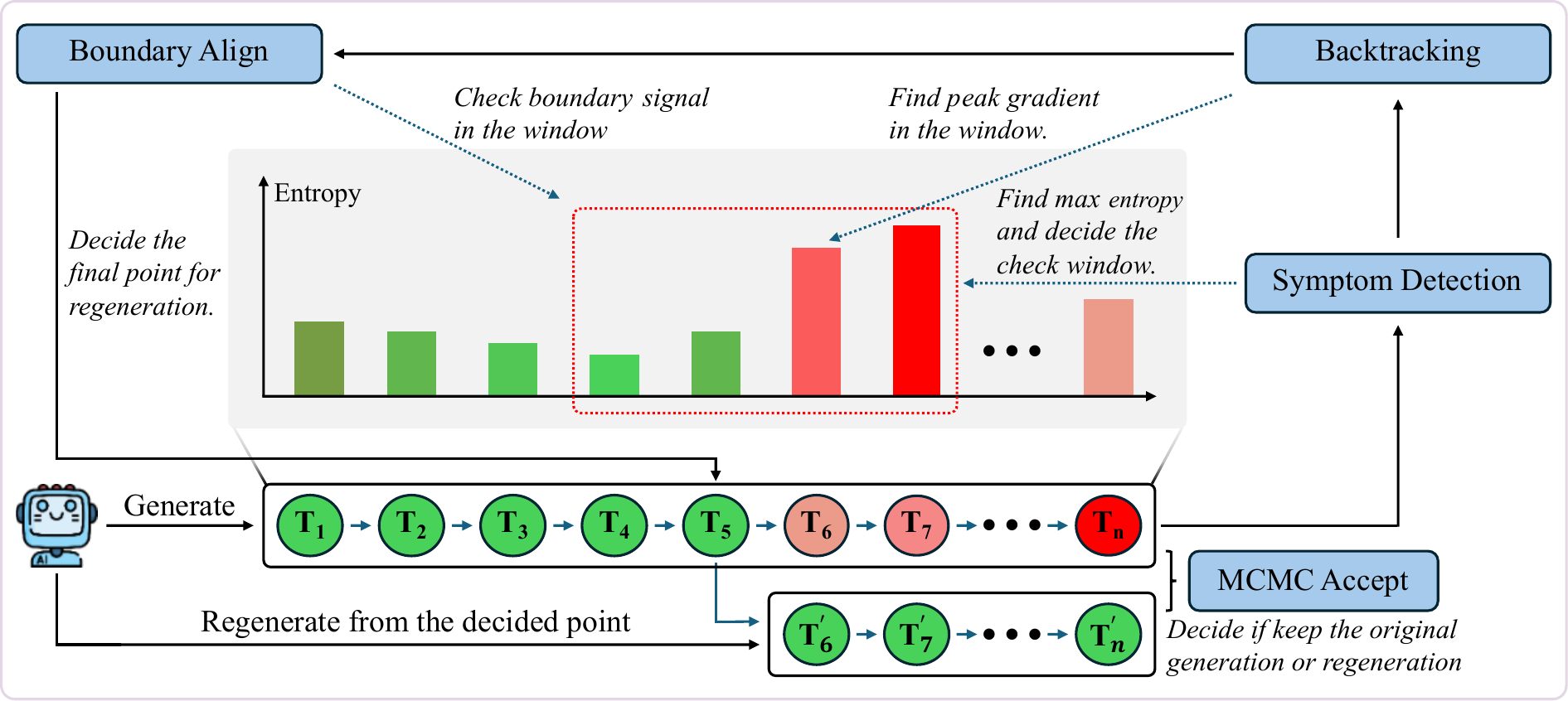}
    \caption{\textbf{REWIND proposal policy.} The entropy peak $t^*$ anchors gradient backtracking to $t_{\text{start}}$. Boundary alignment gives the suffix start $t_{\text{aligned}}$; regeneration and score-based acceptance complete the proposal.}
    \label{fig:2}
\end{figure*}

\subsection{Entropy-Guided Suffix Proposals}
\label{subsec:entropy_guided}
\label{subsubsec:why_entropy}
We use token entropy as the default diagnostic $M_t$ (Equation~\ref{eq:observation_token_entropy}). One revision step has four phases.

\noindent\textbf{Phase 1: Peak Selection.}
Select the highest-entropy generated position:
\begin{equation}
    t^* = \operatorname*{arg\,max}_{L_{\text{context}}<t\leq|\mathbf{x}|} M_t.
\label{eq:rewind_peak}
\end{equation}

\noindent\textbf{Phase 2: Gradient Backtracking.}
Choose the largest difference between adjacent generated positions within a backward distance $w$ of the peak:
\begin{equation}
\begin{aligned}
    \Delta M_t &= M_t-M_{t-1},\quad t>L_{\text{context}}+1,\\
    t_{\text{start}} &= \operatorname*{arg\,max}_{\substack{\max(L_{\text{context}}+2,t^*-w)\leq t\leq t^*}}\Delta M_t.
\end{aligned}
\label{eq:rewind_backtrack}
\end{equation}
If the peak is the first generated token, set $t_{\text{start}}=t^*$ without computing a difference. Thus $w$ specifies a backward distance, with at most $w+1$ candidate positions.

\noindent\textbf{Phase 3: Boundary Alignment.}
To avoid cutting through an incomplete span, search the preceding $k$ tokens for a delimiter:
\begin{equation}
    t_{\text{aligned}}=1+\max\{t\in\mathcal{B}_k(t_{\text{start}})\mid x_t\in\mathcal{S}_{\text{bound}}\},
\label{eq:rewind_boundary}
\end{equation}
where $\mathcal{B}_k(a)=[\max(L_{\text{context}}+1,a-k),a)$ contains integer positions and $\mathcal{S}_{\text{bound}}$ contains tokens whose decoded text includes a delimiter. If none is found, set $t_{\text{aligned}}=t_{\text{start}}$. Keep $\mathbf{x}_{<t_{\text{aligned}}}$, including any selected delimiter, and regenerate the suffix to obtain $\mathbf{x}'$.

\noindent\textbf{Phase 4: MH-Style Acceptance.}
Let $q_{\text{LM}}$ denote the actual next-token proposal distribution. For the shared suffix start $a=t_{\text{aligned}}$, compare both sequences over $m=\min(|\mathbf{x}|,|\mathbf{x}'|)-a+1$ tokens:
\begin{equation}
\begin{aligned}
    \log r_{\text{suffix}} &= S(\mathbf{x}',a;m)-S(\mathbf{x},a;m),\\
    S(\mathbf{z},a;m) &= \sum_{t=a}^{a+m-1}\left[
    \begin{aligned}
    &\alpha\log p_\theta(z_t\mid z_{<t})\\
    &\quad-\log q_{\text{LM}}(z_t\mid z_{<t})
    \end{aligned}\right].
\end{aligned}
\label{eq:acceptance_ratio}
\end{equation}
Accept $\mathbf{x}'$ if $\log u<\min(0,\log r_{\text{suffix}})$ for $u\sim\mathcal{U}(0,1)$; otherwise retain $\mathbf{x}$. This shared-horizon, state-dependent filter is MH-style: it omits forward/reverse site-selection probabilities and does not establish exact MH stationarity. Each subsequent iteration recomputes the entropy trajectory of the retained sequence. Appendix~\ref{app:algorithm} gives the pseudocode.

All signals and sequence scores are computed from the fixed language model, requiring no additional training or external verifier.

\paragraph{Integration with Power Sampling.}
Power Sampling performs $K$ random-cut proposals in each of $B$ blocks; \method{} then adds one entropy-guided proposal to the retained sequence. All proposals share the acceptance score in Equation~\ref{eq:acceptance_ratio}.

\subsection{Theoretical Intuition}
For prefix $s_t=\mathbf{x}_{<t}$ and latent interpretation $Z$, the mixture $p(x\mid s_t)=\sum_z p(z\mid s_t)p(x\mid s_t,z)$ yields $H_t=A_t+B_t$, where $A_t$ is expected within-interpretation entropy and $B_t=I(X_t;Z\mid s_t)$ measures cross-interpretation disagreement. A sharp increase in $B_t$ can create a gradient spike before a later entropy peak. Appendix~\ref{app:appendix_theory} gives sufficient conditions for this local pattern; it motivates candidate selection without guaranteeing a correct retained prefix or causal error localization.

\section{Experiments}
\label{sec:experiments}

\subsection{Experimental Setup} 
\label{subsec:setup}

In the main comparison, we evaluate the effectiveness of our \method{} controller across three model families to assess cross-model robustness. Specifically, we use Qwen2.5-Math-7B \cite{yang2024qwen2} as a representative specialist math model, Llama-3.1-Tülu-3-8B-DPO (Tülu-8B) \cite{lambert2024tulu} as a generalist instruction-tuned baseline, and DeepSeek-1.5B \cite{guo2025deepseek} as a lightweight reasoning-distilled model. All Tülu-8B experiments use this same DPO checkpoint.

We report single-output accuracy on five benchmarks spanning varying difficulty levels: GSM8K (grade-school math) \cite{cobbe2021training}, MATH500 (competition math) \cite{hendrycks2021measuring}, the GPQA Diamond subset (198 questions) \cite{rein2024gpqa}, AIME2025 (frontier math competition) \cite{aops2025aime_i,aops2025aime_ii}, and the BBH formal fallacies task (250 questions) \cite{suzgun2022challenging}. The comparison in Table~\ref{tab:main_results} reports the mean over three independent runs with different random seeds. To further assess robustness under common evaluation protocols and broader settings, Appendix~\ref{app:fewshot_cot} reports standard few-shot CoT prompting, Appendix~\ref{app:additional_generalization} reports code-generation results, and Appendix~\ref{app:multiseed_stability} reports standard deviations for representative Qwen2.5-Math settings.

\paragraph{Baselines.} We compare decoding strategies across four categories. First, \textit{distribution- and decoding-based} methods include Vanilla, Temperature Sampling ($T=0.25$), Typical-p ($p=0.95$) \cite{meister2023locally}, and DoLa \cite{chuang2024dola} (layer-contrastive decoding). Second, \textit{revision-based} methods include Self-Refine \cite{madaan2023selfrefine}, which iteratively regenerates and edits the response based on the model's own feedback. Third, \textit{MCMC-based} sampling is represented by Power Sampling \cite{karan2025reasoning}. Fourth, \textit{candidate selection} includes Best-of-$N$ (BoN@10): we sample $N=10$ responses for each input and select the candidate with the highest mean token log-likelihood.

\paragraph{Implementation.}
We use Power Sampling in 16 blocks with temperature $T=0.25$ and target power $\alpha=4$. Each active block extends the retained sequence and performs $K=10$ random-cut MCMC proposals; \method{} then attempts one additional entropy-guided revision with $w=5$ and $k=10$. All other settings, including prompts, generation limits, and acceptance scoring, match MCMC. Implementation details are provided in Appendix~\ref{app:algorithm}. Our code is available at \url{https://anonymous.4open.science/r/rewind-D310}.

\subsection{Main Results}
\label{subsec:main_results}

\method{} achieves the highest mean accuracy in all 15 model--benchmark settings (Table~\ref{tab:main_results}).

\begin{table}[t]
\centering
\caption{\textbf{Main results.} Accuracy (\%) averaged over three runs. \textbf{Best} and \underline{second-best} are highlighted. Vanilla uses $T=1$.}
\label{tab:main_results}

\vspace{0.2cm}

\resizebox{\columnwidth}{!}{%
\begin{tabular}{l@{\hspace{2pt}} c@{\hspace{5pt}} c@{\hspace{5pt}} c@{\hspace{5pt}} c@{\hspace{5pt}} c}
\toprule
\textbf{Methods} & \textbf{GSM8K} & \textbf{MATH500} & \textbf{GPQA} & \textbf{AIME2025} & \textbf{BBH} \\
\midrule

\multicolumn{6}{c}{\textit{Qwen2.5-Math-7B (Specialist)}} \\
\midrule
Vanilla      & $61.41$ & $48.13$ & $26.77$ & $2.22$  & $63.73$ \\
Temperature  & $75.86$ & $66.80$ & $34.68$ & $7.78$  & $71.20$ \\
Typical-p    & $78.26$ & $67.40$ & $31.31$ & $7.78$  & $74.40$ \\
Self-refine  & $77.63$ & $68.40$ & $32.82$ & $6.67$  & $69.73$ \\
DoLa         & $77.86$ & $70.20$ & $34.84$ & $10.00$ & $73.60$ \\
MCMC   & $79.45$ & $72.60$ & \underline{$35.35$} & $10.00$ & \underline{$75.20$} \\
BoN@10       & \underline{$81.85$} & \underline{$73.00$} & $32.32$ & \underline{$13.33$} & $65.73$ \\
\rowcolor{gray!20}\textbf{\method}& $\textbf{84.31}$ & $\textbf{77.60}$ & $\textbf{35.86}$ & $\textbf{16.67}$ & $\textbf{77.20}$ \\
\addlinespace
\hline
\addlinespace
\multicolumn{6}{c}{\textit{Llama-3.1-Tülu-3-8B-DPO (Generalist)}} \\
\midrule
Vanilla      & $86.27$ & $44.20$ & $24.74$ & $1.11$  & $56.13$ \\
Temperature  & $88.17$ & $44.80$ & $25.25$ & $1.11$  & $56.93$ \\
Typical-p    & $87.67$ & $46.20$ & \underline{$28.79$} & $1.11$  & $54.80$ \\
Self-refine  & $87.87$ & $45.60$ & $24.24$ & $1.11$  & $55.33$ \\
DoLa         & $88.32$ & $45.00$ & $25.76$ & $1.11$  & $57.33$ \\
MCMC   & \underline{$89.08$} & \underline{$50.40$} & $27.27$ & \underline{$2.22$}  & \underline{$57.73$} \\
BoN@10       & $88.40$ & $50.20$ & $26.77$ & $1.11$ & $57.47$ \\
\rowcolor{gray!20}\textbf{\method}& $\textbf{90.90}$ & $\textbf{52.00}$ & $\textbf{29.79}$ & $\textbf{3.33}$  & $\textbf{59.33}$ \\
\addlinespace
\hline
\addlinespace

\multicolumn{6}{c}{\textit{DeepSeek-R1-Distill-Qwen-1.5B (Reasoning)}} \\
\midrule
Vanilla      & $63.84$ & $55.40$ & $27.27$ & $13.33$ & $60.40$ \\
Temperature  & $67.93$ & $57.00$ & $28.79$ & $13.33$ & $62.27$ \\
Typical-p    & $66.56$ & $59.20$ & \underline{$29.80$} & $14.44$ & $64.93$ \\
Self-refine  & $68.51$ & $55.80$ & $27.77$ & $13.33$ & $62.00$ \\
DoLa         & $71.65$ & $58.40$ & \underline{$29.80$} & $14.44$ & \underline{$65.60$} \\
MCMC   & \underline{$73.69$} & $61.20$ & $29.29$ & \underline{$16.67$} & $64.00$ \\
BoN@10       & $70.20$ & \underline{$61.40$} & $28.79$ & \underline{$16.67$} & $65.20$ \\
\rowcolor{gray!20}\textbf{\method}& $\textbf{77.81}$ & $\textbf{64.40}$ & $\textbf{31.81}$ & $\textbf{26.67}$ & $\textbf{66.80}$ \\

\bottomrule
\end{tabular}%
}
\end{table}

For Qwen2.5-Math-7B, \method{} improves MATH500 from 66.80\% under temperature sampling to 77.60\%, and from 72.60\% under standard MCMC to 77.60\% by adding a guided revision after the shared $K=10$ random MCMC steps per block. On DeepSeek-1.5B, \method{} reaches 26.67\% on AIME2025 compared with 16.67\% for MCMC under the same blockwise protocol. Appendix~\ref{app:multiseed_stability}, Table~\ref{tab:multiseed_stability}, expands the corresponding three-run main results with standard deviations.

\subsection{Signal Comparison and Analysis}
\label{subsec:mechanism_analysis}

We evaluate both the choice of trigger signal and the strategy for turning that signal into a revision point. Our \method{} is built on two distinct design choices: (1) backtracking from the uncertainty peak to an earlier entropy-gradient spike (gradient-based intervention selection) and (2) aligning revisions with coarse semantic boundaries. To test whether these choices contribute beyond generic regeneration, Table~\ref{tab:ablation} reports an ablation study on the MATH500 dataset using the Qwen2.5-Math model.

\begin{table}[t]
\centering
\caption{\textbf{Proposal ablation on Qwen2.5-Math-7B/MATH500.} Accuracy (\%) with $K=10$ random MCMC steps per block. Revision variants add a suffix proposal. Local variants use the peak window; ``grad.'' uses its largest entropy increase.}
\label{tab:ablation}
\small
\setlength{\tabcolsep}{3pt}
\begin{tabular}{llr}
\toprule
Variant & Site / alignment & Acc. \\
\midrule
MCMC & Global random / none & 72.60 \\
Random delimiter & Global random / full & 73.20 \\
Peak revision & Peak / full & 74.20 \\
No boundary & Local grad. / none & 75.00 \\
Random in window & Local random / full & 75.20 \\
Restricted boundary & Local grad. / . and newline & 76.80 \\
\method{} & Local grad. / full & 77.60 \\
\bottomrule
\end{tabular}
\end{table}

\paragraph{Comparing Intervention-Point Localization.}
Entropy-gradient backtracking improves accuracy to 77.60\%, compared with 74.20\% for Peak Revision and 72.60\% for random-proposal MCMC. It also outperforms Random-in-Window (75.20\%) by 2.40 percentage points. Random Delimiter and Random-in-Window each add one suffix proposal per block with the same generation and acceptance settings as \method{}; the respective 4.40- and 2.40-point gaps therefore isolate site selection at the same proposal-count budget. These results show that choosing the largest entropy increase within the peak window yields more effective revision sites than revising at the peak or at a random position in the same window.

\paragraph{The Role of Boundary Alignment.}
Boundary alignment increases accuracy from 75.00\% to 77.60\%, a gain of 2.60 percentage points. Random Delimiter revision achieves 73.20\%, showing that boundary cues are most useful when combined with targeted site selection. Restricting the delimiter set to \texttt{.} and \texttt{\textbackslash n} still achieves 76.80\%, only 0.80 points below the full set. These results show that coarse boundary alignment contributes to the gain and remains effective with a small set of delimiters.

\paragraph{Separating Signal Choice from Intervention Strategy.}
For each statistic, we define a revision score $M_t$ before applying peak selection or gradient tracing. Maximum probability, probability gap, and Top-$K$ embedding similarity are reversed, while vocabulary-probability skewness and excess kurtosis are negated. Temporal KL uses $D_{\mathrm{KL}}(p_t\Vert p_{t-1})$. All variants share the suffix-generation and acceptance procedure; Appendix~\ref{app:signal_robustness} gives the exact score definitions. For a fixed $M_t$, peak revision selects its maximum, whereas tracing selects the largest increase $\Delta M_t$ in the preceding window. Table~\ref{tab:main_metric_strategy} compares these strategies under six signal configurations. Tracing improves observed MATH500 accuracy for all six, with gains ranging from 0.2 points for skewness to 3.4 for entropy. Temporal KL also gains 3.2 points, supporting the utility of tracing beyond Shannon entropy in this setting.
\begin{table}[t]
\centering
\caption{\textbf{Signal choice vs. intervention strategy.} MATH500 accuracy (\%) on Qwen2.5-Math-7B for six signal configurations. $\Delta$ is tracing minus peak revision.}
\label{tab:main_metric_strategy}
\small
\setlength{\tabcolsep}{4pt}
\begin{tabular}{lrrr}
\toprule
Signal & Peak & Tracing & $\Delta$ \\
\midrule
Shannon entropy & 74.20 & 77.6 & +3.4 \\
Temporal KL & 71.80 & 75.0 & +3.2 \\
Top-$K$ similarity ($K=5$) & 73.40 & 74.20 & +0.8 \\
Top-$K$ similarity ($K=3$) & 72.60 & 73.40 & +0.8 \\
Skewness & 73.20 & 73.40 & +0.2 \\
Kurtosis & 72.00 & 73.00 & +1.0 \\
\bottomrule
\end{tabular}
\end{table}

\paragraph{Trigger Choice across Benchmarks.}
Among the evaluated trigger configurations, entropy performs best on all three benchmarks as shown in Table~\ref{tab:metric_robustness}. Alternatives show different transfer patterns. Temporal KL improves MATH500 over MCMC (75.00\% vs. 72.60\%), but lowers GSM8K accuracy and ties MCMC on AIME2025. Top-$K$ similarity with $K=5$ and kurtosis also improve MATH500, yet fall below MCMC on AIME2025 (7.78\% and 8.89\% vs. 10.00\%). Skewness gives smaller gains on all three. These results indicate that improving intervention-point selection for a given signal does not necessarily make it a strong trigger across tasks. Together with the within-question observations in Section~\ref{subsec:entropy_observations}, the signal comparisons support entropy as the default trigger for the evaluated controller.

\begin{table}[t]
\centering
\caption{\textbf{Trigger signals across benchmarks.} Accuracy (\%) on Qwen2.5-Math-7B.}
\label{tab:metric_robustness}
\small
\setlength{\tabcolsep}{3pt}
\begin{tabular}{lrrr}
\toprule
Trigger signal & GSM8K & MATH & AIME \\
\midrule
MCMC & 79.45 & 72.60 & 10.00 \\
\midrule
Shannon entropy & \textbf{84.31} & \textbf{77.60} & \textbf{16.67} \\
Skewness & 80.97 & 73.40 & 12.22 \\
Top-$K$ similarity ($K=3$) & 76.19 & 73.40 & 12.22 \\
Temporal KL & 77.86 & 75.00 & 10.00 \\
Kurtosis & 80.59 & 73.00 & 8.89 \\
Top-$K$ similarity ($K=5$) & 79.91 & 74.20 & 7.78 \\
\bottomrule
\end{tabular}
\end{table}

\paragraph{Answer-level intervention utility.}
Across the three main-table runs on MATH500, \method{} averages 46 repairs and 21 regressions per run relative to MCMC, yielding 25 net corrections (+5.00 points; Appendix~\ref{app:intervention_analysis}).

\subsection{Inference Budget and BoN Comparison}
\label{subsec:cost}
Relative to single-pass temperature sampling, the reported Qwen2.5-Math generation-token ratios are $8.84\times$ for MCMC and $9.26\times$ for \method{}. BoN@10 samples ten independent responses per input, corresponding to a nominal $\sim10\times$ generation-token budget under comparable average response lengths. On this nominal basis, the three methods operate at roughly comparable generation budgets, with \method{} adding approximately 4.8\% generated tokens over MCMC.

Table~\ref{tab:main_results} reports the accuracy comparison with likelihood-ranked BoN@10. \method{} achieves higher accuracy on all 15 reported model--benchmark pairs, with gains of 1.60--11.47 percentage points. On Qwen2.5-Math-7B/MATH500, it reaches 77.60\% versus 73.00\% for BoN@10 (+4.60 points) and on DeepSeek-1.5B/AIME2025, it reaches 26.67\% versus 16.67\% (+10.00 points). These comparisons favor targeted suffix repair over the evaluated candidate-selection baseline.

\subsection{Transfer across Tasks and Prompts}
\label{subsec:transfer}
\begin{table}[t]
\centering
\caption{\textbf{Transfer across tasks and prompts.} Accuracy (\%) on HumanEval and under 8-shot CoT math prompting. The latter uses Qwen2.5-Math-7B.}
\label{tab:main_prompt_task_transfer}
\small
\setlength{\tabcolsep}{4pt}
\begin{tabular}{lrrr}
\toprule
Model / benchmark & Temp. & MCMC & \method{} \\
\midrule
\multicolumn{4}{l}{\textit{HumanEval}} \\
Qwen2.5-Math-7B & 52.4 & 56.1 & 58.5 \\
Tülu-8B & 58.5 & 62.8 & 66.5 \\
\midrule
\multicolumn{4}{l}{\textit{8-shot CoT}} \\
MATH500 & 68.4 & 73.2 & 78.8 \\
GSM8K & 85.5 & 86.3 & 88.4 \\
AIME2025 & 7.8 & 10.0& 16.7 \\
\bottomrule
\end{tabular}
\end{table}

\paragraph{Code Generation.}
Table~\ref{tab:main_prompt_task_transfer} extends the evaluation to HumanEval~\cite{chen2021codex}, where \method{} gains 2.4 points over MCMC on Qwen2.5-Math-7B and 3.7 on Tülu-8B. Restricting Qwen2.5-Math-7B's boundary set to \texttt{.} and \texttt{\textbackslash n} retains 57.90\% accuracy, compared with 58.5\% for the full set and 56.1\% for MCMC (Appendix~\ref{app:boundary_robustness}). Most of the observed gain thus survives with two general-purpose boundary anchors. These gains support transfer beyond mathematics on two backbones; Appendix~\ref{app:additional_generalization} provides the full comparison.

\paragraph{Few-Shot Prompting.}
With standard 8-shot CoT prompts on Qwen2.5-Math-7B, Table~\ref{tab:main_prompt_task_transfer} shows gains over MCMC of 5.6 points on MATH500, 2.1 on GSM8K, and 6.7 on AIME2025. The gains therefore persist in this additional prompting setting and Appendix~\ref{app:fewshot_cot} reports all decoding baselines.

\section{Discussion}

\paragraph{From Entropy Diagnostics to Revision Sites.}
The repeated-generation study shows that entropy carries useful information about response correctness. Mean token entropy provides the strongest response-level discrimination in Table~\ref{tab:entropy_diagnostics}, while the token-wise entropy trajectory preserves the positional information needed for revision. \method{} uses an entropy peak to anchor a local search and selects the largest entropy increase within that window. This connects the entropy--correctness finding to an actionable revision policy.

\paragraph{Signal Choice and Intervention Strategy.}
The signal comparisons and ablations identify complementary ingredients of effective revision. Entropy achieves the highest accuracy in the main cross-benchmark signal sweep, and tracing improves over peak revision for all six tested signal configurations on MATH500. Boundary alignment further improves the entropy-based controller and remains effective with only periods and newlines. Together, these results suggest that useful refinement depends on both an informative uncertainty signal and a suitable rule for converting that signal into a suffix proposal.

\paragraph{When Revision Fails.}
Manual inspection of Qwen2.5-Math-7B on MATH500 finds confident errors without a salient entropy spike in 7.6\% of all problems and detected but unrecovered failures in 14.8\% (Appendix~\ref{app:residual_failure}). The latter account for most residual failures, highlighting regeneration quality as a remaining bottleneck.

\section{Conclusion}
\label{sec:conclusion}
In this paper, we present empirical evidence that token-level entropy is associated with reasoning errors. We found that incorrect responses have higher mean token entropy on 88.12\% of paired MATH500 questions, with an average within-question relative increase of 112\%. Based on this association between the during-generation signal and reasoning errors, we introduce \methodfull{}, a training-free controller that uses entropy-gradient backtracking to select revision points, aligns them with natural boundaries, and filters regenerated suffixes with an MH-style score. Across three model families and five benchmarks, \method{} outperforms all baselines, improving over random-proposal MCMC by 0.51--10.00 percentage points. Signal comparisons identify entropy as the strongest tested trigger across three mathematical benchmarks, and controlled ablations support the contributions of both gradient backtracking and boundary alignment. Together, these results show that intrinsic uncertainty can guide targeted trajectory repair and improve test-time reasoning.

\clearpage
\section*{Limitations}

\label{sec:limitations}

\paragraph{Computational Cost.} 
The cost study in Section~\ref{subsec:cost} reports generation-token counts, which exclude entropy-only forward passes and do not establish normalized FLOPs or wall-clock cost. Entropy computation and discarded proposals add overhead, and efficiency comparisons depend on implementation and batching.

\paragraph{Scope and Generalization.}
The task and model coverage in Section~\ref{subsec:transfer} is limited. Supplementary HumanEval results cover only two backbones. Further tasks, backbones, and repeated runs are needed to assess broader transfer.

\paragraph{Signal Validity and Statistical Coverage.}
The benchmarks provide no token-level first-error annotations, so the evidence concerns intervention utility rather than localization accuracy. AIME2025 has only 30 questions, limiting the precision of accuracy comparisons. Finally, $7.6\%$ of the inspected MATH500 examples are confidently wrong without a salient entropy spike (Appendix~\ref{app:residual_failure}), setting a direct limit on uncertainty-triggered repair.

\section*{Ethical Considerations}

REWIND can still produce incorrect answers despite improved benchmark accuracy. Its iterative revision requires additional inference computation, which may increase energy consumption.

\bibliography{example_paper}

\clearpage
\appendix

\section*{Appendix Overview}

The appendix provides supporting material for the main empirical and methodological claims. Appendix~\ref{app:appendix_theory} gives a theoretical account of entropy-gradient localization, while Appendix~\ref{app:theoretical_justification} explains the sequence-level motivation behind the accept/reject refinement loop. Appendix~\ref{app:qwen72b_compare} compares targeted refinement with a larger single-pass model. Appendices~\ref{app:boundary_robustness}--\ref{app:supplementary_robustness} report robustness checks for boundary alignment, trigger signals, and supplementary controller variants. Appendices~\ref{app:intervention_analysis}--\ref{app:residual_failure} analyze repairs, regressions, proposal acceptance, intervention-site coverage, and residual failures. Finally, Appendix~\ref{app:fewshot_cot}, Appendix~\ref{app:additional_generalization}, and Appendix~\ref{app:multiseed_stability} provide additional prompting, generalization, and stability results, followed by qualitative case studies in Appendix~\ref{app:case_studies} and implementation details and pseudocode in Appendix~\ref{app:algorithm}.

\section{Entropy-Gradient Intuition}
\label{app:appendix_theory}

We analyze entropy-gradient localization through \emph{semantic fragmentation}: the model becomes uncertain between conflicting interpretations, increasing cross-interpretation disagreement. We derive sufficient conditions under which the resulting entropy-gradient spike identifies the local onset of fragmentation.

\subsection{Setup and Notation}

Let $s_t := x_{<t}$ denote the prefix at step $t$. The next-token entropy is defined as
\begin{equation}
    H_t := H(X_t \mid s_t) = -\sum_{v} p(v \mid s_t) \log p(v \mid s_t),
\end{equation}
and the entropy gradient is $\Delta H_t := H_t - H_{t-1}$.

\subsection{Main Result}

\begin{theorem}[Entropy Gradient Localizes Semantic Fragmentation]
\label{thm:main_appendix}
Assume the model's predictive distribution admits a latent decomposition
\begin{equation}
    p(x \mid s_t) = \sum_{z \in \mathcal{Z}} p(z \mid s_t) \cdot p(x \mid s_t, z),
\end{equation}
where $z$ represents implicit semantic interpretations. Define the within-world entropy $A_t := \mathbb{E}_{Z}[ H(X_t \mid s_t, Z)]$ and the cross-world disagreement $B_t := I(X_t; Z \mid s_t)$.

\textbf{(Decomposition)} The entropy decomposes as
\begin{equation}
    H_t = A_t + B_t, \quad \text{hence} \quad \Delta H_t = \Delta A_t + \Delta B_t.
\end{equation}

\textbf{(Localization)} If $t_0$ lies inside the searched neighborhood $\mathcal{W}(t^*)$ and semantic fragmentation occurs at $t_0$, causing a mutual information jump $\Delta B_{t_0} \geq \Gamma$, and the within-world entropy is locally smooth with $|\Delta A_t| \leq L$ over the searched neighborhood $\mathcal{W}(t^*)$, then
\begin{equation}
    \Delta H_{t_0} \geq \Gamma - L.
\end{equation}

\textbf{(Identification)} If additionally $\Delta B_t \leq \gamma$ for all competing positions $t \in \mathcal{W}(t^*)$, $t \neq t_0$ (no comparable MI jumps in the local search window) and $\Gamma - \gamma > 2L$, then
\begin{equation}
    t_0 = \operatorname*{arg\,max}_{t \in \mathcal{W}(t^*)} \Delta H_t.
\end{equation}
\end{theorem}

\begin{proof}
\textbf{Part 1} follows from the definition of conditional mutual information. Recall that $I(X_t; Z \mid s_t) = H(X_t \mid s_t) - H(X_t \mid s_t, Z)$. Rearranging gives $H(X_t \mid s_t) = H(X_t \mid s_t, Z) + I(X_t; Z \mid s_t)$. Since $H(X_t \mid s_t, Z) = \sum_z p(z \mid s_t) H(X_t \mid s_t, z) = \mathbb{E}_Z[H(X_t \mid s_t, Z)] = A_t$ and $I(X_t; Z \mid s_t) = B_t$, we obtain $H_t = A_t + B_t$.

\textbf{Part 2}: At $t_0$, we have $\Delta H_{t_0} = \Delta A_{t_0} + \Delta B_{t_0}$. By the smoothness assumption, $\Delta A_{t_0} \geq -L$. By the fragmentation assumption, $\Delta B_{t_0} \geq \Gamma$. Thus $\Delta H_{t_0} \geq -L + \Gamma = \Gamma - L$.

\textbf{Part 3}: For any competing position $t \in \mathcal{W}(t^*)$, $t \neq t_0$, we have $\Delta A_t \leq L$ (local smoothness) and $\Delta B_t \leq \gamma$ (no comparable MI jumps). Thus $\Delta H_t = \Delta A_t + \Delta B_t \leq L + \gamma$. The condition $\Gamma - L > L + \gamma$ (equivalently, $\Gamma - \gamma > 2L$) ensures $\Delta H_{t_0} \geq \Gamma - L > L + \gamma \geq \Delta H_t$ for all such $t$.
\end{proof}

\begin{remark}[Interpretation of Assumptions]
The theorem combines a latent mixture model, locally smooth within-interpretation entropy, and an isolated mutual-information jump. The mixture model aligns with Bayesian interpretations of in-context learning in which latent concepts govern generation~\citep{xie2021explanation}. Smoothness is required within the controller's search window. The signal-strength condition $\Gamma-\gamma>2L$ ensures that the fragmentation event has a larger entropy gradient than competing positions in that window.
\end{remark}

\subsection{Why an Earlier Cut May Help}
Backtracking from $t^*$ to $t_0<t^*$ removes the intervening tokens and lets the model regenerate from before the identified distributional shift. When those tokens encode a harmful premise, the earlier cut provides an opportunity to revise it. The benefit depends on the content removed; proposal-policy ablations and paired final-answer outcomes evaluate this choice empirically.

\subsection{When Does Fragmentation Produce a Detectable Signal?}

\begin{lemma}[MI Lower Bound]
\label{lem:mi}
Suppose at step $t$, two semantic worlds $z_1, z_2$ have total mass $m=p(z_1 \mid s_t)+p(z_2 \mid s_t)$, normalized conditional weights $\tilde{\pi}_1,\tilde{\pi}_2 \geq \eta$ within this two-world subset, and $\mathrm{TV}(p(\cdot \mid s_t, z_1), p(\cdot \mid s_t, z_2)) \geq \delta$ (they disagree on predictions). Then
\begin{equation}
    I(X_t; Z \mid s_t) \geq 2m\,\eta(1-\eta)\delta^2.
\end{equation}
\end{lemma}

\begin{proof}[Proof Sketch]
Let $p_i = p(\cdot \mid s_t, z_i)$ denote the conditional distributions. We first condition on the event $Z \in \{z_1,z_2\}$, whose probability is $m$, and write the normalized weights as $\tilde{\pi}_i=p(z_i\mid s_t)/m$. For the two-component conditional mixture $\bar{p}_{12}=\tilde{\pi}_1 p_1+\tilde{\pi}_2 p_2$, the total variation between each component and the conditional mixture satisfies $\mathrm{TV}(p_1,\bar{p}_{12})=\tilde{\pi}_2 \mathrm{TV}(p_1,p_2)\geq \tilde{\pi}_2\delta$ and similarly for $p_2$.

Applying Pinsker's inequality $D_{\mathrm{KL}}(p \| q) \geq 2\,\mathrm{TV}(p,q)^2$ within this conditional two-world mixture yields a conditional mutual-information lower bound of $2\tilde{\pi}_1\tilde{\pi}_2\delta^2$. Multiplying by the subset probability $m$ gives the stated contribution to the full mutual information by the chain rule and non-negativity of the remaining components.

Since $\tilde{\pi}_1,\tilde{\pi}_2 \geq \eta$ and $\tilde{\pi}_1+\tilde{\pi}_2=1$, we have $\tilde{\pi}_1\tilde{\pi}_2 \geq \eta(1-\eta)$, yielding $I(X_t; Z \mid s_t) \geq 2m\eta(1-\eta)\delta^2$.
\end{proof}

\begin{remark}
The lemma links disagreement between plausible interpretations to a larger mutual-information component. Temporal localization follows when this component exhibits an isolated jump satisfying Theorem~\ref{thm:main_appendix}. The quadratic dependence on $\delta$ captures the strength of the prediction disagreement.
\end{remark}

\subsection{Scope of the Analysis}

The analysis describes errors that manifest as localized semantic fragmentation. \emph{Confident errors} can remain low-entropy because the model commits to a wrong but internally coherent interpretation. \emph{Gradual drift} spreads uncertainty changes across steps, weakening the localized-jump signal.

The condition $\Gamma-\gamma>2L$ characterizes when the fragmentation signal dominates background changes. Large uncertainty shifts during legitimate reasoning transitions can compete with that signal and affect the selected revision site.

The latent decomposition serves as an explanatory model. Our experiments use observable entropy traces and assess revision utility through final-answer accuracy.

\subsection{Sequence-Level Score Motivation}
\label{app:theoretical_justification}
Temperature $1/\alpha$ normalizes each next-token distribution separately. Because these normalizers depend on the prefix, their product varies across trajectories. Sequence-level Power Sampling instead uses the global normalizer in $p(\mathbf{x})^\alpha/Z(\alpha)$, motivating suffix-level scoring. In exact MH, this normalizer cancels from the forward/reverse ratio. \method{} compares suffix scores at the selected site using the MH-style filter in Equation~\ref{eq:acceptance_ratio}; state-dependent site-selection probabilities are omitted. Final-answer accuracy measures the resulting refinement performance.

\section{Comparison with Qwen2.5-Math-72B}
\label{app:qwen72b_compare}

We compare iterative refinement with Qwen2.5-Math-7B against single-pass sampling from Qwen2.5-Math-72B at $T=0.25$. The two settings differ in both model size and inference budget, evaluating a trade-off between model scale and iterative computation. Table~\ref{tab:qwen72b_compare} reports accuracy across all five benchmarks.

\begin{table*}[t]
\centering
\caption{\textbf{Comparison with Qwen2.5-Math-72B.} Accuracy (\%) is reported across five benchmarks. The 72B baseline is single-pass sampling with $T=0.25$, while \method{} uses the Qwen2.5-Math-7B backbone with iterative refinement.}
\label{tab:qwen72b_compare}
\vspace{0.2cm}
\begin{appendixtableblock}[\textwidth]
\begin{tabular}{lccccc}
\toprule
\textbf{Method} & \textbf{GSM8K} & \textbf{MATH500} & \textbf{GPQA} & \textbf{AIME2025} & \textbf{BBH} \\
\midrule
Qwen2.5-Math-72B, single-pass ($T=0.25$) & $\textbf{94.01}$ & $75.40$ & $33.33$ & $6.67$ & $67.60$ \\
\rowcolor{gray!20}\method{} (Qwen2.5-Math-7B) & $84.31$ & $\textbf{77.60}$ & $\textbf{35.86}$ & $\textbf{16.67}$ & $\textbf{77.20}$ \\
\bottomrule
\end{tabular}
\end{appendixtableblock}
\end{table*}

\method{} with the 7B backbone outperforms the single-pass 72B baseline on MATH500, GPQA, AIME2025, and BBH; the 72B model achieves higher accuracy on GSM8K. These results demonstrate the utility of spending additional inference computation on targeted revision in these settings.

\section{Boundary Alignment Robustness}
\label{app:boundary_robustness}

Boundary alignment uses a lightweight decoded-token delimiter rule. We test its sensitivity by restricting the delimiter set to two general-purpose anchors, \texttt{\textbackslash n} and \texttt{.}. Table~\ref{tab:boundary_robustness} shows that this restricted variant preserves most of the full method's performance on both MATH500 and HumanEval. The results support coarse boundary alignment and show that the gains remain strong with a small delimiter set.

\begin{table}[t]
\centering
\caption{\textbf{Robustness to restricted delimiters.} The restricted setting keeps only \texttt{\textbackslash n} and \texttt{.}; full uses \texttt{\textbackslash n}, \texttt{.}, \texttt{,}, \texttt{\$}, and \texttt{)}.}
\label{tab:boundary_robustness}
\vspace{0.2cm}
\begin{appendixtableblock}
\begin{tabular}{lccc}
\toprule
\textbf{Benchmark} & \textbf{MCMC} & \textbf{Restricted} & \textbf{Full} \\
\midrule
MATH500 & $72.60$ & $76.80$ & $\textbf{77.60}$ \\
HumanEval & $56.10$ & $57.90$ & $\textbf{58.50}$ \\
\bottomrule
\end{tabular}
\end{appendixtableblock}
\end{table}

\section{Signal Robustness}
\label{app:signal_robustness}
\label{subsec:signal_robustness}
This appendix supplements the signal analysis in Section~\ref{subsec:mechanism_analysis}. We first define the implemented revision scores, then report the full peak-versus-tracing comparison in Table~\ref{tab:metric_strategy}, with the MCMC reference. The cross-benchmark signal sweep is reported in main-text Table~\ref{tab:metric_robustness}.

\paragraph{Signal Definitions.}
Let $\ell_t$ be the model logits predicting generated token $x_t$, $p_t=\operatorname{softmax}(\ell_t)$ the full-vocabulary distribution before sampling transformations, and $V=|\mathcal{V}|$. All logarithms are natural. Write $p_{t,(1)}\geq p_{t,(2)}$ for its two largest probabilities. The following scores $M_t$ are passed to the common peak-selection and gradient-tracing rules:
\begin{center}
\small
\setlength{\tabcolsep}{4pt}
\renewcommand{\arraystretch}{1.2}
\begin{tabular}{@{}ll@{}}
\toprule
Statistic & Revision score $M_t$ \\
\midrule
Shannon entropy & $H(p_t)=-\sum_v p_t(v)\log p_t(v)$ \\
Temporal KL & $D_{\mathrm{KL}}(p_t\Vert p_{t-1})$ \\
Max probability & $1-p_{t,(1)}$ \\
Probability gap & $1-(p_{t,(1)}-p_{t,(2)})$ \\
Top-$K$ similarity & $1-C_{t,K}$ \\
Skewness & $-V^{-1}\sum_v z_t(v)^3$ \\
Kurtosis (excess) & $-[V^{-1}\sum_v z_t(v)^4-3]$ \\
\bottomrule
\end{tabular}
\end{center}
Temporal KL is explicitly
\begin{equation}
M_t=\sum_{v\in\mathcal{V}}p_t(v)
\log\frac{p_t(v)}{p_{t-1}(v)},
\end{equation}
with score zero at the first generated position. Maximum probability and probability gap are reversed so that lower confidence gives a higher revision score.

For Top-$K$ similarity, let $v_{t,1},\ldots,v_{t,K}$ be the highest-probability tokens and $e_v$ the fixed model's input embedding of token $v$. We use the unweighted mean pairwise cosine similarity:
\begin{equation}
C_{t,K}=\frac{2}{K(K-1)}
\sum_{i<j}\cos(e_{v_{t,i}},e_{v_{t,j}}).
\end{equation}
Thus lower embedding similarity gives a higher score. We evaluate $K\in\{3,5\}$.

Skewness and excess kurtosis are calculated across the $V$ vocabulary probability values. Following the implementation's sample-standard-deviation convention, define
\begin{equation}
\begin{aligned}
s_t^2&=\frac{1}{V-1}\sum_{v\in\mathcal{V}}
\bigl(p_t(v)-1/V\bigr)^2,\\
z_t(v)&=\frac{p_t(v)-1/V}{s_t+10^{-10}}.
\end{aligned}
\end{equation}
Both statistics are negated before site selection. For each score, peak revision maximizes $M_t$, while tracing maximizes the signed difference $M_t-M_{t-1}$ within the preceding peak window. The suffix-generation and acceptance procedure is shared across signal variants.

\paragraph{Varentropy.}
The varentropy configuration in Table~\ref{tab:supp_trigger_metrics} uses the temperature-scaled full-vocabulary distribution
\begin{equation}
\widetilde p_t=\operatorname{softmax}(\ell_t/T_s),
\qquad T_s=0.25.
\end{equation}
Its score is the variance of token surprisal under $\widetilde p_t$:
\begin{equation}
M_t=\sum_v\widetilde p_t(v)[\log\widetilde p_t(v)]^2
-H(\widetilde p_t)^2.
\end{equation}
The implementation takes the absolute value of this difference to obtain its nonnegative score. The other scores defined above use the unscaled distribution $p_t$.

\paragraph{Robustness of the Gradient Tracing Strategy.}
\label{subsubsec:strategy_universality}
To test whether the gradient-versus-peak intervention pattern is specific to Shannon entropy, we decouple the \textit{localization strategy} from the \textit{metric choice}. On MATH500, we compare two revision strategies across multiple distributional statistics (KL divergence, Top-$K$ similarity, skewness, and kurtosis). Peak selection revises at the position with maximum metric value ($t^* = \arg\max M_t$). In contrast, \textbf{Targeted Gradient Tracing (Tracing)} revises at the maximum metric gradient within a preceding window ($t_{\text{start}} = \arg\max \Delta M_t$), treating the sharp distributional shift as a proxy for an earlier useful intervention point.

\begin{table}[t]
\centering
\caption{\textbf{Peak vs. gradient-based intervention localization.} Accuracy (\%) on MATH500 under naive peak revision and gradient tracing.}
\label{tab:metric_strategy}
\vspace{0.2cm}
\resizebox{0.95\linewidth}{!}{%
\begin{tabular}{lcc@{\hspace{5pt}}c}
\toprule
\textbf{Metric} & \textbf{Naive} & \textbf{Tracing} & \textbf{$\Delta$ Gain} \\
\midrule
\textit{MCMC} & $72.6$ & - & - \\
\midrule
\rowcolor{gray!20}\textbf{Shannon Entropy} & $74.2$ & $\textbf{77.6}$ & $\textbf{+3.4}$ \\
KL Divergence (Temporal) & $71.8$ & $75.0$ & $+3.2$ \\
Top-K Similarity ($K=5$) & $73.4$ & $74.2$ & $+0.8$ \\
Top-K Similarity ($K=3$) & $72.6$ & $73.4$ & $+0.8$ \\
Skewness & $73.2$ & $73.4$ & $+0.2$ \\
Kurtosis & $72.0$ & $73.0$ & $+1.0$ \\
\bottomrule
\end{tabular}%
}
\end{table}

On Qwen2.5-Math-7B/MATH500, gradient tracing improves accuracy for every listed configuration in Table~\ref{tab:metric_strategy}. The largest gain is 3.4 points for entropy, followed by 3.2 for temporal KL. Kurtosis (+1.0) and skewness (+0.2) show smaller gains. These comparisons support revising at an earlier distributional shift within the peak window.

Entropy achieves the highest accuracy on all three benchmarks in Table~\ref{tab:metric_robustness}. Both Top-$K$ configurations gain 0.8 points from tracing on MATH500 in Table~\ref{tab:metric_strategy}, while entropy remains stronger in the cross-benchmark comparison.

\section{Supplementary Controller Robustness}
\label{app:supplementary_robustness}

We evaluate alternative trigger metrics and boundary/window settings over three random seeds.

Maximum probability and probability gap use the reversed scores defined in Appendix~\ref{app:signal_robustness}. Varentropy uses the temperature-scaled distribution with $T_s=0.25$, as specified there.

\begin{table*}[t]
\centering
\caption{\textbf{Supplementary trigger-metric checks.} Max probability, probability gap, and varentropy are added as common confidence-style baselines under the same MCMC revision budget. Entropy values are repeated from Table~\ref{tab:metric_robustness} for reference.}
\label{tab:supp_trigger_metrics}
\vspace{0.2cm}
\begin{appendixtableblock}[\textwidth]
\begin{tabular}{lcc}
\toprule
\textbf{Trigger metric ($M_t$)} & \textbf{MATH500} & \textbf{AIME2025} \\
\midrule
\textit{MCMC} & $72.60$ & $10.00$ \\
\midrule
\rowcolor{gray!20}\textbf{Shannon Entropy (main)} & $\textbf{77.60}$ & $\textbf{16.67}$ \\
Max Probability & $70.60$ & $\textbf{16.67}$ \\
Probability Gap & $71.80$ & $13.33$ \\
Varentropy & $70.60$ & $13.33$ \\
\bottomrule
\end{tabular}
\end{appendixtableblock}
\end{table*}

In Table~\ref{tab:supp_trigger_metrics}, maximum probability matches entropy on AIME2025 and falls below MCMC on MATH500. Probability gap and varentropy also trail entropy on MATH500.

\begin{table*}[t]
\centering
\caption{\textbf{Supplementary boundary and window checks.} Expanded delimiters add additional punctuation and closing brackets to the default set. The adaptive-window variant uses the full delimiter set and allows a larger look-back window capped at 20 tokens.}
\label{tab:supp_boundary_window}
\vspace{0.2cm}
\begin{appendixtableblock}[\textwidth]
\begin{tabular}{lcc}
\toprule
\textbf{Setting} & \textbf{MATH500} & \textbf{AIME2025} \\
\midrule
\rowcolor{gray!20}\textbf{Full delimiters, fixed $w=5$ (main)} & $\textbf{77.60}$ & $\textbf{16.67}$ \\
Expanded delimiters, fixed $w=5$ & $74.00$ & $13.33$ \\
Full delimiters, adaptive window (max 20) & $73.60$ & $10.00$ \\
\bottomrule
\end{tabular}
\end{appendixtableblock}
\end{table*}

The default delimiter set and fixed $w=5$ window achieve the highest accuracy in Table~\ref{tab:supp_boundary_window}. Both expanded delimiters and the adaptive window reduce accuracy relative to this setting, supporting a small delimiter set and a fixed local search window.

\section{Answer-Level Intervention Analysis}
\label{app:intervention_analysis}

\begin{table}[t]
\centering
\caption{\textbf{Answer-level repair and regression analysis on MATH500.} Mean counts across the same three runs summarized in Table~\ref{tab:main_results}. Repair/regression is the ratio of the mean counts.}
\label{tab:repair_regression}
\vspace{0.2cm}
\begin{appendixtableblock}
\begin{tabular}{lccc}
\toprule
\textbf{Setting} & \textbf{\shortstack{Mean\\repairs}} & \textbf{\shortstack{Mean\\regressions}} & \textbf{Repair / regression} \\
\midrule
Qwen2.5-Math-7B on MATH500 & $46$ & $21$ & $2.19\times$ \\
\bottomrule
\end{tabular}
\end{appendixtableblock}
\end{table}

Table~\ref{tab:repair_regression} compares final answers under standard MCMC and \method{} over the same three runs used for the main results in Table~\ref{tab:main_results}. Within each run, we count repairs and regressions over the 500 MATH500 questions, then average these counts across runs. Relative to MCMC, \method{} repairs 46 previously incorrect answers and regresses on 21 previously correct answers on average per run. The mean net gain is therefore 25 correct answers per 500 questions, corresponding to the 5.00-percentage-point difference in mean accuracy reported in the main table.

\begin{table}[t]
\centering
\caption{\textbf{Proposal acceptance ratios.} Percentage of proposals accepted under random MCMC and \method{} across three benchmarks.}
\label{tab:acceptance_ratio_analysis}
\vspace{0.2cm}
\begin{appendixtableblock}
\begin{tabular}{lccc}
\toprule
\textbf{Benchmark} & \textbf{MCMC} & \textbf{\method{}} & \textbf{$\Delta$} \\
\midrule
AIME2025 & $44.4$ & $39.4$ & $-5.0$ \\
MATH500 & $57.3$ & $54.1$ & $-3.2$ \\
GSM8K & $65.6$ & $63.4$ & $-2.2$ \\
\bottomrule
\end{tabular}
\end{appendixtableblock}
\end{table}

\method{} has lower proposal acceptance on all three benchmarks in Table~\ref{tab:acceptance_ratio_analysis}. On MATH500, the complete controller accepts 54.1\% of proposals, compared with 57.3\% for random MCMC, alongside the positive net answer gain in Table~\ref{tab:repair_regression}.

\paragraph{MATH500 observation protocol.}
The repeated-generation observations in Section~\ref{subsec:entropy_observations} use temperature $1.0$ and base seed 42, with a 3,072-token generation cap and a 4,096-token context limit. Of the 10,000 responses, 248 reach their generation cap.

\section{Intervention-Site Coverage}
\label{app:multi_error_coverage}

Each guided proposal selects one aligned cut point and regenerates its suffix. We count the distinct aligned positions visited across refinement rounds under the blockwise protocol with $K=10$ random MCMC steps per block. On MATH500, the controller visits an average of \textbf{3.18} distinct sequence positions per problem (Table~\ref{tab:multi_error_coverage}).

\begin{table}[t]
\centering
\caption{\textbf{Iterative coverage across refinement rounds on MATH500.} Mean number of distinct aligned sequence positions visited per problem during refinement.}
\label{tab:multi_error_coverage}
\vspace{0.2cm}
\begin{appendixtableblock}
\begin{tabular}{lcc}
\toprule
\textbf{Setting} & \shortstack{\textbf{Random MCMC}\\\textbf{steps / block}} & \textbf{Distinct positions / problem} \\
\midrule
Qwen2.5-Math-7B on MATH500 & $10$ & $3.18$ \\
\bottomrule
\end{tabular}
\end{appendixtableblock}
\end{table}

This statistic summarizes the controller's exploration of cut points across refinement rounds.

\section{Residual Failure Analysis}
\label{app:residual_failure}

We manually classify the remaining MATH500 failures using their entropy traces and regeneration outcomes. Table~\ref{tab:failure_taxonomy} summarizes two categories: (i) confident errors, where the model is wrong without a salient entropy spike, and (ii) detected-but-unrecovered cases, where \method{} identifies instability but the regenerated suffix still does not reach a correct answer.

\begin{table}[t]
\centering
\caption{\textbf{Failure taxonomy on MATH500 (manual inspection).} Confident errors are a minority of all evaluated problems; most residual failures occur after the entropy spike is detected but regeneration still does not recover a correct continuation.}
\label{tab:failure_taxonomy}
\vspace{0.2cm}
\begin{appendixtableblock}
\begin{tabular}{lcc}
\toprule
\textbf{Failure type} & \textbf{Share of all problems} & \textbf{Share of failures} \\
\midrule
Confident errors & $7.6\%$ & $33.9\%$ \\
Detected but unrecovered & $14.8\%$ & $66.1\%$ \\
\bottomrule
\end{tabular}
\end{appendixtableblock}
\end{table}

Detected-but-unrecovered cases account for 66.1\% of residual failures. This distribution motivates stronger suffix proposals for detected instability and complementary signals for confident errors.

\section{Standard Few-Shot CoT Evaluation}
\label{app:fewshot_cot}

Our main experiments use a single-output setting to isolate the effect of the inference-time controller. To test whether the gains persist under common prompting protocols, we additionally evaluate Qwen2.5-Math-7B with standard 8-shot CoT prompting on the math benchmarks. Table~\ref{tab:fewshot_cot} shows that \method{} continues to improve over vanilla decoding, temperature sampling, and standard MCMC under this setting.

\begin{table*}[t]
\centering
\caption{\textbf{Performance under standard 8-shot CoT prompting.} Accuracy (\%) is reported for Qwen2.5-Math-7B. \method{} maintains consistent gains under few-shot CoT prompts.}
\label{tab:fewshot_cot}
\vspace{0.2cm}
\begin{appendixtableblock}[\textwidth]
\begin{tabular}{lcccc}
\toprule
\textbf{Benchmark} & \textbf{Vanilla} & \textbf{Temperature} & \textbf{Standard MCMC} & \textbf{\method{}} \\
\midrule
AIME2025 & $3.3$ & $7.8$ & $10.0$ & $\textbf{16.7}$ \\
MATH500 & $51.2$ & $68.4$ & $73.2$ & $\textbf{78.8}$ \\
GSM8K & $74.9$ & $85.5$ & $86.3$ & $\textbf{88.4}$ \\
\bottomrule
\end{tabular}
\end{appendixtableblock}
\end{table*}

\section{Additional Generalization Results}
\label{app:additional_generalization}

Table~\ref{tab:human_eval_generalization} reports supplementary HumanEval~\cite{chen2021codex} results for Qwen2.5-Math-7B and Tülu-8B. These results show gains beyond mathematics on both backbones.

\begin{table*}[t]
\centering
\caption{\textbf{Supplementary HumanEval results.} \method{} also improves code-generation performance on two backbones.}
\label{tab:human_eval_generalization}
\vspace{0.2cm}
\begin{appendixtableblock}[\textwidth]
\begin{tabular}{lcccc}
\toprule
\textbf{Model} & \textbf{Vanilla} & \textbf{Temperature} & \textbf{MCMC} & \textbf{\method{}} \\
\midrule
Qwen2.5-Math-7B & $31.1$ & $52.4$ & $56.1$ & $\textbf{58.5}$ \\
Tülu-8B & $51.2$ & $58.5$ & $62.8$ & $\textbf{66.5}$ \\
\bottomrule
\end{tabular}
\end{appendixtableblock}
\end{table*}

\section{Multi-Seed Stability}
\label{app:multiseed_stability}

Table~\ref{tab:multiseed_stability} expands the Qwen2.5-Math-7B results in Table~\ref{tab:main_results} for MATH500 and AIME2025 using the same three runs. We report the arithmetic mean and sample standard deviation of run-level accuracies, with the latter computed using denominator $R-1$ for $R=3$ runs.

\begin{table}[t]
\centering
\caption{\textbf{Multi-seed stability on Qwen2.5-Math-7B.} Accuracy mean $\pm$ sample standard deviation for the same three runs summarized in Table~\ref{tab:main_results}.}
\label{tab:multiseed_stability}
\vspace{0.2cm}
\begin{appendixtableblock}
\begin{tabular}{lcc}
\toprule
\textbf{Method} & \textbf{MATH500} & \textbf{AIME2025} \\
\midrule
Vanilla & $48.13 \pm 0.64$ & $2.22 \pm 1.92$ \\
Temperature & $66.80 \pm 0.35$ & $7.78 \pm 1.92$ \\
MCMC & $72.60 \pm 2.03$ & $10.00 \pm 3.33$ \\
\rowcolor{gray!20}\method{} & $\textbf{77.60} \pm 1.56$ & $\textbf{16.67} \pm 3.33$ \\
\bottomrule
\end{tabular}
\end{appendixtableblock}
\end{table}

\section{Case Studies}
\label{app:case_studies}

We present illustrative examples demonstrating how \method{} identifies and corrects problematic reasoning patterns.

\subsection{Case Study 1: Methodological Improvement}

\textbf{Problem:} Find the sum of all integer bases $b > 9$ for which $17_b$ is a divisor of $97_b$.

\vspace{0.5em}
\noindent\textbf{Before Revision (Brute-force Enumeration):}

The model initially attempts to solve the problem by exhaustively checking each value of $k$:

\begin{quote}
\textit{``For $k = 2$, we get $b(9-2) = 7(2-1)$, which simplifies to $7b = 7$, but this does not satisfy $b > 9$. For $k = 3$, we get $b(9-3) = 7(3-1)$, which simplifies to $6b = 14$, and $b = \frac{14}{6}$, which is not an integer. For $k = 4$... For $k = 5$... For $k = 6$... For $k = 7$... For $k = 8$...''}
\end{quote}

This approach is tedious, error-prone, and lacks mathematical elegance.

\vspace{0.5em}
\noindent\textbf{After Revision (Modular Arithmetic):}

After detecting high entropy at the beginning of the solution, the model revises to use a more principled approach:

\begin{quote}
\textit{``We rewrite $9b + 7$ in terms of $b + 7$:}
\[9b + 7 = 9(b + 7) - 63 + 7 = 9(b + 7) - 56\]
\textit{Thus: $9b + 7 \equiv -56 \pmod{b + 7}$. For $b + 7$ to divide $9b + 7$, it must divide $-56$. Therefore, $b + 7$ must be a divisor of 56: $\{1, 2, 4, 7, 8, 14, 28, 56\}$.''}
\end{quote}

\noindent\textbf{Improvement:} The revision transforms a brute-force enumeration into an elegant number-theoretic solution, reducing the problem to finding divisors of 56.

\subsection{Case Study 2: Geometric Error Correction}

\textbf{Problem (2025 AIME I, Problem 2).} In triangle $ABC$, the points along $\overline{AB}$ occur in the order $A,D,E,B$, with $(AD,DE,EB)=(4,16,8)$. Along $\overline{AC}$, the order is $A,F,G,C$, with $(AF,FG,GC)=(13,52,26)$. Point $F$ is the midpoint of $\overline{DM}$, and point $E$ is the midpoint of $\overline{GN}$, defining the reflected points $M$ and $N$. Given that quadrilateral $DEGF$ has area $288$, determine the area of heptagon $AFNBCEM$~\cite{aops2025aime_i}.

\vspace{0.5em}
\noindent\textbf{Before Revision (Coordinate Error):}

The model makes a critical error in setting up coordinates:

\begin{quote}
\textit{``$F$ is on $\overline{AC}$ such that $AF = 13$ and $FG = 52$, so $F$ is at $(13, 0)$.''}
\end{quote}

\textbf{Error:} Point $F$ should lie on line $AC$, not on the $x$-axis (which represents line $AB$). This fundamental mistake would propagate through the entire solution.

\vspace{0.5em}
\noindent\textbf{After Revision (Side-Aligned Coordinates):}

The revised excerpt assigns points to the two side directions:

\begin{quote}
\textit{``Determine the total length of side $AB$: $AB = AD + DE + EB = 4 + 16 + 8 = 28$.\\
Determine the total length of side $AC$: $AC = AF + FG + GC = 13 + 52 + 26 = 91$.\\
Find the coordinates:
\begin{itemize}
    \item $D$ is at $(4, 0)$ on $AB$
    \item $E$ is at $(20, 0)$ on $AB$
    \item $F$ is at $(0, 13)$ on $AC$
    \item $G$ is at $(0, 65)$ on $AC$
\end{itemize}''}
\end{quote}

\noindent\textbf{Improvement:} The revision correctly distinguishes points on $AB$ from points on $AC$.

\noindent\textbf{Coordinate convention and area scaling.} With $A$ as the origin, let $\mathbf{e}_{AB}$ and $\mathbf{e}_{AC}$ be unit vectors along the two sides. We interpret a coordinate pair $(u,v)$ in the excerpt as the position vector $u\mathbf{e}_{AB}+v\mathbf{e}_{AC}$ in this oblique basis. If $\phi=\angle BAC$, Euclidean areas equal coordinate areas multiplied by $\sin\phi$. For $DEGF$, the coordinate area and the scale implied by its given Euclidean area are
\[
\begin{aligned}
S_{\mathrm{coord}}(DEGF)&=\frac{20\cdot65-4\cdot13}{2}=624,\\
\sin\phi&=\frac{288}{624}=\frac{6}{13}.
\end{aligned}
\]

\paragraph{Case Study Summary.}

These illustrative excerpts show changes in the reasoning text:
\begin{enumerate}
    \item \textbf{Improve methodology:} Transform inefficient enumeration into elegant mathematical solutions (Case 1)
    \item \textbf{Correct errors:} Identify and fix fundamental setup mistakes before they propagate (Case 2)
\end{enumerate}

The controlled final-answer comparisons in Section~\ref{subsec:main_results} quantify the overall accuracy gains from revision.

\section{Implementation Details and Algorithm}
\label{app:algorithm}

\subsection{Benchmark Subsets and Prompts}
\label{app:evaluation_protocol}

The five main benchmarks use the subsets in Table~\ref{tab:benchmark_protocol}. Each run evaluates the full listed subset, and accuracy is computed over questions within that run. AIME2025 combines all 15 questions from AIME I and all 15 from AIME II. BBH refers specifically to the \texttt{formal\_fallacies} task, and GPQA uses the Diamond subset.

\begin{table}[t]
\centering
\caption{\textbf{Evaluation subsets and BoN generation limits.} $n$ is the number of questions per run. The last column gives the maximum new tokens per BoN candidate.}
\label{tab:benchmark_protocol}
\begin{appendixtableblock}
\begin{tabular}{llrr}
\toprule
\textbf{Benchmark} & \textbf{Subset} & $\boldsymbol{n}$ & \textbf{BoN limit} \\
\midrule
GSM8K & Test & 1,319 & 2,048 \\
MATH500 & MATH test subset & 500 & 3,072 \\
GPQA & Diamond & 198 & 3,072 \\
AIME2025 & AIME I + II & 30 & 4,096 \\
BBH & \texttt{formal\_fallacies} & 250 & 3,072 \\
\bottomrule
\end{tabular}
\end{appendixtableblock}
\end{table}

\paragraph{Mathematical prompts and chat formatting.}
The zero-shot formatter for MATH500, GSM8K, and AIME2025 uses the following template, where \texttt{[QUESTION]} is replaced with the benchmark question:
\begin{quote}\small
Can you solve the following math problem? \texttt{[QUESTION]} Please reason step by step, and put your final answer within \texttt{\textbackslash boxed\{\}}.
\end{quote}
Qwen2.5-Math-7B and DeepSeek-1.5B receive this text directly. For Tülu-8B, we apply the tokenizer's \texttt{apply\_chat\_template} to the same text as a single user message, with \texttt{tokenize=False} and \texttt{add\_generation\_prompt} set to \texttt{True}. The resulting prompt has a \texttt{<|user|>} turn followed by a \texttt{<|assistant|>} generation prefix, with no added system message. The additional 8-shot experiments are reported in Appendix~\ref{app:fewshot_cot}.

\subsection{\texorpdfstring{Best-of-$N$}{Best-of-N} Sampling and Ranking}
\label{app:bon_protocol}

BoN@10 samples ten responses independently for each input at temperature $T=1.0$. Candidate generation uses top-$k=50$ and the checkpoint's nucleus-sampling setting: top-$p=1.0$ for Qwen2.5-Math-7B, $0.90$ for Tülu-8B (the DPO checkpoint), and $0.95$ for DeepSeek-1.5B. Typical sampling is disabled ($\texttt{typical\_p}=1.0$). Each candidate stops at EOS or the generation limit in Table~\ref{tab:benchmark_protocol}. For BBH and GPQA on Qwen2.5-Math-7B, the candidate limit is further bounded by $4096-L_{\text{prompt}}$ to fit the model's context length.

We decode each candidate and tokenize its response text separately from the prompt for teacher-forced rescoring. For response tokens $\mathbf{y}^{(i)}=(y^{(i)}_1,\ldots,y^{(i)}_{L_i})$ and formatted prompt $q$, selection uses
\begin{equation}
\begin{aligned}
\ell_i &= \frac{1}{L_i}\sum_{t=1}^{L_i}
\log p_\theta(y^{(i)}_t\mid q,\mathbf{y}^{(i)}_{<t}),\\
i^* &= \operatorname*{arg\,max}_{1\leq i\leq10}\ell_i.
\end{aligned}
\label{eq:bon_ranking}
\end{equation}
Scoring uses the model's original conditional probabilities before temperature scaling or top-$k$/top-$p$ filtering. Prompt, padding, and EOS tokens are excluded from both the score and its length normalization. Ties are resolved by the original candidate order. Empty responses are excluded from ranking, and candidates are not filtered by answer validity or correctness. Only the selected response is used for the reported single-output accuracy.

\subsection{Controller Implementation}

\paragraph{Sampling and stopping.}
Initial extensions and suffix proposals both use temperature $T=0.25$, with target power $\alpha=1/T=4$. Random and guided proposals use the same shared-horizon score in Equation~\ref{eq:acceptance_ratio}. Generation stops early at EOS. Prompts, initialization, block partitioning, generation limits, sampling settings, and acceptance scoring are shared with the MCMC baseline.

\paragraph{Boundary matching.}
A token belongs to $\mathcal{S}_{\text{bound}}$ if its decoded text contains a newline or one of \texttt{.}, \texttt{,}, \texttt{\$}, and \texttt{)}. We select the nearest matching token within the $k=10$ generated tokens preceding the traced position and begin regeneration after that token, retaining the entire matched token in the prefix. If no matching token is found, regeneration begins at the traced position. Appendix~\ref{app:supplementary_robustness} compares fixed and adaptive windows.

\paragraph{Software and hardware.}
The controller uses PyTorch and Hugging Face Transformers on NVIDIA H800 GPUs.

\paragraph{Code release and licensing.}
Experimental code is available in the repository linked in Section~\ref{subsec:setup}. We release implementation code only and have verified that the applicable upstream licenses permit this release. We do not redistribute pretrained model weights or benchmark datasets.

Algorithm~\ref{alg:entropy_revise} gives one proposal step. Repeated steps recompute the entropy trajectory of the retained sequence; iteration and block budgets are specified in the experimental setup.

\begin{algorithm}[t]
\caption{One Entropy-Guided Revision Step}
\label{alg:entropy_revise}
\begin{algorithmic}[1]
\REQUIRE Sequence $\mathbf{x}$, prompt length $L_{\text{context}}$, backward distance $w$, boundary distance $k$, delimiter tokens $\mathcal{S}_{\text{bound}}$, distributions $p_\theta,q_{\text{LM}}$, power $\alpha$
\ENSURE Updated sequence
\STATE \textbf{Phase 1: Peak Selection}
\STATE Compute entropy $H_t$ for $L_{\text{context}}<t\leq|\mathbf{x}|$
\STATE $t^*\gets\operatorname*{arg\,max}_{L_{\text{context}}<t\leq|\mathbf{x}|}H_t$
\STATE \textbf{Phase 2: Gradient Backtracking}
\IF{$t^*=L_{\text{context}}+1$}
    \STATE $t_{\text{start}}\gets t^*$
\ELSE
    \STATE $\mathcal{T}\gets[\max(L_{\text{context}}+2,t^*-w),t^*]$
    \STATE $\Delta H_t\gets H_t-H_{t-1}$ for $t\in\mathcal{T}$
    \STATE $t_{\text{start}}\gets\operatorname*{arg\,max}_{t\in\mathcal{T}}\Delta H_t$
\ENDIF
\STATE \textbf{Phase 3: Boundary Alignment}
\STATE $\mathcal{C}\gets\{t\in\mathcal{B}_k(t_{\text{start}}):x_t\in\mathcal{S}_{\text{bound}}\}$
\STATE $a\gets1+\max\mathcal{C}$ if $\mathcal{C}\ne\varnothing$; otherwise $a\gets t_{\text{start}}$
\STATE Generate $\mathbf{x}'_{\text{suffix}}\sim q_{\text{LM}}(\cdot\mid\mathbf{x}_{<a})$
\STATE $\mathbf{x}'\gets\mathbf{x}_{<a}\oplus\mathbf{x}'_{\text{suffix}}$
\STATE \textbf{Phase 4: MH-Style Acceptance}
\STATE $m\gets\min(|\mathbf{x}|,|\mathbf{x}'|)-a+1$
\STATE $\log r_{\text{suffix}}\gets S(\mathbf{x}',a;m)-S(\mathbf{x},a;m)$
\STATE Sample $u\sim\mathcal{U}(0,1)$
\IF{$\log u<\min(0,\log r_{\text{suffix}})$}
    \STATE \textbf{return} $\mathbf{x}'$
\ELSE
    \STATE \textbf{return} $\mathbf{x}$
\ENDIF
\end{algorithmic}
\end{algorithm}

\end{document}